\documentclass[3p,onecolumn]{elsarticle}

\usepackage{lineno,hyperref}
\modulolinenumbers[5]

\journal{Theoretical Computer Science}

\usepackage{enumerate}
\usepackage{amsmath}
\usepackage{amscd,amssymb,graphics,lineno}
\usepackage{mathrsfs} 
\usepackage{mathtools, nccmath}
\usepackage{stmaryrd}
\usepackage{mathrsfs,verbatim} 

\usepackage{calc}

\usepackage{amsthm}

\newtheorem{theorem}{Theorem}[section]
\newtheorem{corollary}[theorem]{Corollary} 
\newtheorem{lemma}[theorem]{Lemma}
\newtheorem{proposition}[theorem]{Proposition}
\theoremstyle{definition}

\theoremstyle{remark}
\newtheorem{remark}[theorem]{Remark}

\newtheorem{example}[theorem]{Example}
\numberwithin{equation}{section}
\newtheorem{definition}[theorem]{Definition}

\newcommand{\R}{{\mathbb R}}
\newcommand{\I}{{\mathbb I}}

\newcommand{\E}{{\mathbb E}}

\newcommand{\N}{{\mathbb N}}

\newcommand{\ve}{{\varepsilon}}
 
\def\norm #1{{\left\Vert\,#1\,\right\Vert}}
\newcommand{\abs}[1]{\lvert#1\rvert}
\def\fr(#1/#2){{^{\mbox{$_#1$}}\!/_{#2}}}
\font\mathf=cmex10
\def\va{\hbox{\mathf\char'76}}
\def\pipe{\hbox{${\raise.28em\va\atop\raise.50em\va}$}}

\begin{document}

\begin{frontmatter}

\title{A learning problem whose consistency is equivalent to the non-existence of real-valued measurable cardinals
}

\author[mymainaddress,mysecondaryaddress]{Vladimir G. Pestov
}

\address[mymainaddress]{Instituto de Matem\'atica e Estat\'\i stica, Universidade Federal da Bahia,
Ondina, 40.170-115,
Salvador, BA, Brasil \footnote{Professor Visitante Titular.}}
\address[mysecondaryaddress]{Department of Mathematics and Statistics,
       University of Ottawa,
       Ottawa, ON K1N 6N5, Canada \footnote{Emeritus Professor.}}
\ead{vpest283@uottawa.ca}

\begin{abstract}
We show that the $k$-nearest neighbour learning rule is universally consistent in a metric space $X$ if and only if it is universally consistent in every separable subspace of $X$ and the density of $X$ is less than every real-measurable cardinal. In particular, the $k$-NN classifier is universally consistent in every metric space whose separable subspaces are sigma-finite dimensional in the sense of Nagata and Preiss if and only if there are no real-valued measurable cardinals. 
The latter assumption is relatively consistent with ZFC, however the consistency of the existence of such cardinals cannot be proved within ZFC. 
Our results were inspired by an example sketched by C\'erou and Guyader in 2006 at an intuitive level of rigour.
\end{abstract}


\end{frontmatter}

\section{Introduction}
The default model of statistical learning assumes that datapoints belong to a standard Borel space, whose measurable sigma-algebra is generated by a complete separable metric, and the learning rule -- the mapping associating a hypothesis to every sample -- is measurable in some sense. However, it certainly makes sense to push the limits of the model by dropping some of the restrictions and studying the consequences.
An interesting recent work by Ben-David, Hrube{\v s}, Moran, Shpilka, and Yehudayoff \cite{BHMSY1} (see \cite{BHMSY2} for a more detailed exposition) illustrates what happens if the requirement of measurability of the learning rule is dropped. In this case, there is a learning problem -- the Expectation Maximization (EMX) problem -- whose consistency in the Euclidean domain is equivalent to a version of the Continuum Hypothesis, and thus independent of the standard axioms ZFC of the Zermelo--Fraenkel set theory with the Axiom of Choice. That a solution to the EMX problem cannot be Borel measurable, was proved by Hart \cite{hart}. (In the Appendix below, we show that such a solution cannot even be Lebesgue measurable.) 
Thus, the independence of the EMX learning problem of ZFC could be an artefact of a model allowing non-measurable learning rules.

In this work, we arrive at a somewhat similar phenomenon by relaxing a different assumption on the learning model, that of separability of the domain. Our  problem is the classical $k$-nearest neighbour supervised learning rule in a metric space. However, we do not require the metric spaces to be separable, that is, admitting a dense countable subset. In this setting, the $k$-NN learning rule, in order to be univerally consistent, requires an additional assumption on top of the ZFC: that of the non-existence of real-valued measurable cardinals. It is consistent with ZFC to assume such cardinals do not exist, but it cannot be proved within ZFC that their existence is consistent. (It is still theoretically possible that there is a proof they do not exist.) It is worth stressing that our results are already meaningful for learning in subspaces of such a ``mundane'' metric space as $\ell^{\infty}$, the space of all bounded sequences of reals with the supremum distance.

The $k$-nearest neighbour classification rule is well studied for finite-dimensional Euclidean spaces, where especially important classical contributions have been made in \cite{CH} and \cite{stone}. While there has been some recent work on the $k$-NN classifier in separable metric spaces \cite{CG,CKP}, there has been virtually no research in the non-separable setting. 
The only attempt we know of, is an informal example discussed by C\'erou and Guyader in \cite{CG} at an intuitive level of rigour. The present article has grown out of our attempt to understand the example and explore the idea on a more rigorous mathematical footing. 

The article \cite{CG} has been a major advance in the theory of the $k$-NN classifier in separable metric spaces. In Section 2.2, the authors argue that the separability assumption on the metric space is necessary, and to make their point, they propose the following example. We reproduce it {\em verbatim} from the article.

\begin{quote}
``Let us define a distance d on [0, 1] as follows:
\[d(x,x^\prime)=\begin{cases}
0&\mbox{ if }x=x^\prime \\
1&\mbox{ if }xx^\prime=0\mbox{ and }x\neq x^\prime\\
2&\mbox{ if }xx^\prime\neq 0\mbox{ and }x\neq x^\prime.
\end{cases}\]
Since the triangle inequality holds, $d$ is a distance on $[0,1]$. But $([0,1],d)$ is clearly not separable. 

Let $\mu$ be the simple distribution defined as follows: with probability one half, one picks the origin $0$; with probability one half, one picks a point uniformly in $[0,1]$. Mathematically speaking, if $\lambda_{[0,1]}$ denotes the Lebesgue measure on $[0,1]$ and $\delta_0$ the Dirac measure at the origin:
\[\mu=\frac 12 \delta_0 + \frac 12 \lambda_{[0,1]}.\]
The way to attribute a label to a point $x\in [0,1]$ is deterministic: if $x=0$ then $y=0$; if $0<x\leq 1$ then $y=1$. As $Y$ is a deterministic function of $X$, the Bayes risk $L^{\ast}$ is equal to zero. Nevertheless, it is intuitively clear that the asymptotic probability of error with the nearest neighbors rule does not converge to $0$:
\[\lim_{n\to\infty} \E [L_n]=\frac 12 >0.\]
So the nearest neighbors classifier is not weakly consistent in this context, although we are in finite dimension.''
\end{quote}

Thus, the metric $d$ defined on the interval $[0,1]$ is discrete: every point of $(0,1]$ is at a distance $1/2$ from zero, and the pairwise distances between distinct points in $(0,1]$ are all equal to one. According to the Law of Large Numbers, a random $n$-sample will contain $n/2\pm\ve$ datapoints equal to zero. As $n\to\infty$ and $k/n\to 0$, with high confidence, an overwhelming proportion among the $k$ nearest neighbours of any point $x\in [0,1]$ will be equal to zero, so labelled $0$, and the majority vote will return the value $0$ for the predicted label of $x$. In the limit $n\to\infty$, the classifier returned by the $k$-nearest neighbour rule will converge in probability to the classifier identically taking value $0$ at every point. The probability of misclassification of a random point $X\sim \mu$ will converge to $1/2$. 

Now, some criticism of the example. First,
it is not very clear in what sense the metric space $([0,1],d)$ is of ``finite dimension.'' This is actually a very interesting issue, to be discussed later on. But more immediate is the observation that the Lebesgue measure is only defined on the sigma-algebra of Lebesgue measurable sets, while the Borel sigma-algebra generated by the discrete metric $d$ contains all subsets of $[0,1]$. In what sense then to understand the uniform measure $\lambda$? This measure, and the entire example, would only make sense if the Lebesgue measure can be extended over the sigma-algebra of all subsets of the unit interval.

The possibility of such an extension was one of the questions asked early on in measure theory, soon after Lebesgue's early work. 
For an excellent account of the subsequent developments, we refer to Fremlin's survey article \cite{fremlin}. See also \cite{jech}, especially Chapter 10.

The {\em Banach--Ulam problem} asked for a description of all probability measures on the sigma-algebra $2^X$ of all subsets of a given set $X$ in terms enabling one to decide, in case $X=\R$ (or $X=[0,1]$, since we will be talking about probability measures), whether they are extensions of the Lebesgue measure. 
A cardinal $\tau$ is {\em real-valued measurable} if there exists a probability measure on the sigma-algebra of all subsets of $\tau$ such that all singletons have measure zero, and the ideal of null sets is closed under unions of $<\tau$ of its members. According to Ulam's theorem \cite{ulam}, if such a measure space is in addition non-atomic, then $\tau$ does not exceed the cardinality of the continuum, and in this case $\tau$ is called an {\em atomlessly measurable} cardinal. 
If there exists a measure on $\tau$ as above having atoms, then $(\tau,2^{\tau})$ admits a probability measure vanishing on the singletons, whose ideal of null sets is closed under unions of $<\tau$ elements, and which only takes two values, $0$ and $1$. In this case the cardinal $\tau$ is called {\em measurable} (or {\em two-valued measurable}). Every such cardinal is (much) greater than the continuum.

Another result by Ulam \cite{ulam} states that the Lebesgue measure admits an extension over the sigma-algebra of all subsets of $\R$ if and only if there exists an atomlessly measurable cardinal. And it is relatively consistent with ZFC to assume that the real-valued measurable cardinals do not exist (that is, if the system ZFC is consistent, then the system ZFC + ``there are no real-valued measurable cardinals'' is consistent as well). It is a consequence of the following two results. Again according to Ulam \cite{ulam}, every real-valued measurable cardinal $\tau$ is weakly inaccessible, that is, an uncountable regular limit cardinal (in other words, the cofinality of $\tau$ is equal to $\tau$, and $\tau=\aleph_{\tau}$). At the same time, it is consistent with ZFC to assume that weakly inaccessible cardinals do not exist (see e.g. \cite{HJ}, page 279).
Thus, it is consistent with ZFC to assume that the Lebesgue measure cannot be extended to a sigma-additive measure on all subsets of the real line.
Furthermore, it is not possible to prove, within ZFC, that the existence of the real-valued measurable cardinals is relatively consistent with ZFC (see Th. 12.12 in \cite{jech}), and the same applies to the atomlessly measurable and two-valued measurable cardinals separately. In fact, the statements about the existence of a real-valued measurable cardinal, an atomlessly-measurable cardinal, and a two-valued measurable cardinal are all equiconsistent according to a theorem by Solovay (\cite{solovay}; see also Th. 22.1 in \cite{jech}), that is, if ZFC plus any one of them is relatively consistent, then ZFC plus all of them is relatively consistent as well.
The possibility of a proof that such cardinals do not exist has not been discarded, however, quoting from \cite{fremlin}, ``at present almost no-one is seriously searching for a proof in ZFC that real-valued measurable cardinals do not exist.''

Thus, the above example of C\'erou and Guyader remains valid under the assumption of existence of an atomlessly measurable cardinal. 
On the other hand, assuming there are no such cardinals, every Borel probability measure on a metric space of cardinality continuum will be necessarily supported on a separable subspace (as follows e.g. from Th. 11.10 in \cite{GP}), in which case it is easy to see the $k$-NN classifier in the example will be consistent. We conclude that the example of C\'erou and Guyader holds if and only if there exists an atomlessly measurable cardinal.

Those ideas go further than a single example.
Let us remind that a learning rule is {\em universally consistent} if for any distribution of the set of labelled points, as $n\to\infty$, the misclassification error of the rule converges to the minimal possible error, the Bayes error. (The exact definitions are given later in the article.) The most important single result about the $k$-NN classifier is probably the 1977 theorem of Stone \cite{stone} asserting that, under the assumptions $k,n\to\infty$, $k/n\to 0$, the $k$-NN classification rule in the finite-dimensional Euclidean space $\R^d$ is universally consistent. Thanks to \cite{CG} and \cite{preiss83}, this result has been extended to a wide subclass of separable metric spaces.

Now we are able to convert the above idea of C\'erou and Guyader into the following central result of the article.
\vskip .2cm
\noindent
{\bf Theorem \ref{th:central}.} {\em
Let $\Omega$ be a metric space. The $k$-nearest neighbour classifier is universally consistent in $\Omega$ if and only if it is universally consistent in every separable subspace of $\Omega$, and the density of $\Omega$ is less than the smallest real-valued measurable cardinal.}
\vskip .2cm

The proof in the case where $d(\Omega)\geq\tau$, where $\tau$ is a two-valued measurable cardinal, is based on a Ramsey property of such cardinals (which allows to extract from $\Omega$ a metric subspace of large cardinality whose metric takes a constant value between different points), followed by a modification of the argument by C\'erou and Guyader. In the case of atomlessly measurable cardinal $\tau$, the argument is different, and is based on the following technical observation, of interest on its own. 
\vskip .2cm
\noindent
{\bf Theorem \ref{th:maharam_type}.} {\em
Let $(\Omega,{\mathcal A},\mu)$ be a probability space admitting a learning rule consistent for every measurable regression function. Then the Maharam type of the measure algebra $({\mathcal A},\mu\circ\Delta)$ does not exceed the cardinality of $\Omega$.}
\vskip .2cm

Now one concludes by the Gitik--Shelah theorem \cite{GH}.

One may argue that the problem disappears if we only consider regular probability measures on the metric spaces, in the sense that every Borel set is approximated from within to any given measure with compact sets. However, there is no logical reason to impose this restriction. The model of learning makes perfect sense in the most general measurable space (as we discuss below), and a specification to the case of Borel spaces generated by a metric does not impose any regularity conditions.

Finally, it is very interesting to examine the meaning in which the interval $[0,1]$ with a discrete metric ``is of finite dimension.'' Clearly, it should be a concept of dimension of a metric space relevant for the consistency of the $k$-NN classifier. And such a concept indeed exists.
C\'erou and Guyader \cite{CG} have in particular shown, with the help of results of Preiss \cite{preiss83}, that the $k$-NN classification rule is universally consistent in a separable metric space whenever it is {\em sigma-finite dimensional,} in the following sense.
The {\em Nagata dimension} \cite{nagata} of a subset $X$ in a metric space $\Omega$ is less than or equal to $\delta$ if there is a value $s>0$ (a scale) such that every finite family of closed balls in $\Omega$ with centres in $X$ and of radii $<s$ admits a subfamily of multiplicity $\leq\delta+1$ that covers the centres of all the original balls. A metric space $\Omega$ is sigma-finite dimensional if it can be covered by countably many subsets each having finite Nagata dimension in $\Omega$. 

We now have:
\vskip .2cm
\noindent
{\bf Corollary \ref{c:sigma-finite}.} {\em
Let $\Omega$ be a metric space each of whose separable subspaces is sigma-finite dimensional in the sense of Nagata and Preiss. Then the $k$-NN classifier is universally consistent in $\Omega$ if and only if the density of $\Omega$ is strictly less than the smallest real-valued measurable cardinal.}
\vskip .2cm

Getting back to the example of C\'erou and Guyader, it is easy to see that it has Nagata dimension zero on the scale $1/2$. Another interesting example is the metric hedgegog $J(\tau)$ of spininess $\tau$, that is, the union of $\tau$ copies of the unit interval $[0,1]$, glued at the left endpoint, and equipped with the geodesic metric. The hedgehog has infinite Nagata dimension, however it is sigma-finite dimensional, so the $k$-NN classifier is universally consistent in this space if and only if $\tau$ is less than any real-valued measurable cardinal.

All this of course vindicates the vision of C\'erou and Guyader in the sense that, whether or not the real-valued measurable cardinals exist, it only makes sense to study the universal consistency of the $k$-NN classifier in metric spaces in the separable case. If the density of the space is smaller than each measurable cardinal, the consistency of the rule is fully determined by what happens in the separable subspaces. If not, the $k$-NN classifier will not be universally consistent no matter what.

\section*{Acknowledgement}
The author wants to thank Samuel Gomes da Silva for his suggestion to present the work of Ben-David, Hrube{\v s}, Moran, Shpilka, and Yehudayoff \cite{BHMSY1,BHMSY2} as a minicourse at the Logic, Set Theory, and Topology Week of the 2020 Summer School at the Federal University of Bahia. The first part of the minicourse covered the central result of the above authors, and an extended write-up of the second part forms the present article.

\section{Learning rules}

\subsection{Learning in a measurable space\label{subs:learning}}
Let $\Omega=(\Omega,{\mathcal A})$ be a measurable space, that is, a set equipped with a sigma-algebra of subsets $\mathcal A$. The product $\Omega\times \{0,1\}$ now becomes a measurable space in a natural way. The elements $x\in \Omega$ are known as {\em unlabelled points}, and elements $(x,y)\in \Omega\times \{0,1\}$ are {\em labelled points.} A finite sequence of labelled points, $\sigma=(x_1,x_2,\ldots,x_n,y_1,y_2,\ldots,y_n)\in \Omega^n\times \{0,1\}^n$, is a {\em labelled sample.}

A {\em classifier} in $X$ is a mapping
\[T\colon \Omega\to \{0,1\},\]
assigning a label to every point. The mapping is usually assumed to be measurable (or, more generally, universally measurable, that is, measurable with regard to the intersection of all possible completions of the sigma-algebra). 
This assumption is necessary in order for things like the misclassification error to be well defined, although some authors are allowing for non-measurable maps, working with the outer measure instead.

Let $\tilde\mu$ be a probability measure defined on the measurable space $\Omega\times \{0,1\}$. Denote $(X,Y)$ a random element of $\Omega\times \{0,1\}$ following the law $\tilde\mu$. The misclassification error of $T$ is the quantity
\begin{align*}
\mbox{err}_{\tilde\mu}(T) &= \tilde\mu \{(x,y)\in \Omega\times \{0,1\}\colon T(x)\neq y\} \\
&= P[T(X)\neq Y].
\end{align*}

A {\em learning rule} in $(\Omega,{\mathcal A})$ is a mapping, $\mathcal L$, that, when shown a labelled sample, $\sigma$, produces a classifier. In other words, a learning rule determines a label of a point $x$ on the basis of a labelled learning sample $\sigma$:
\[{\mathcal L}\colon\bigcup_{n=1}^{\infty}\Omega^n\times \{0,1\}^n\times \Omega \ni (\sigma,x)\mapsto {\mathcal L}(\sigma)(x) \in \{0,1\}.\]
Again, the map above is usually assumed to be (universally) measurable with regard to the product sigma-algebra. 

The labelled datapoints are modelled by a sequence of independent, identically distributed random elements $(X_n,Y_n)\in \Omega\times \{0,1\}$ following the law $\tilde\mu$. For each $n$, the {\em misclassification error} of the rule $\mathcal L$ restricted to $\Omega^n\times \{0,1\}^n$ (which we denote ${\mathcal L}_n$) is the value
\begin{align*}
\mbox{err}_{\tilde\mu}{\mathcal L}_n & = (\tilde\mu^n\otimes \tilde\mu)\{(\sigma,x,y)\colon {\mathcal L}_n(\sigma)(x)\neq y\} \\
&= P\left[ {\mathcal L}_n(\sigma)(X)\neq Y\right] \\
&= \E_{\sigma\sim\tilde\mu^n}\mbox{err}_{\tilde\mu}({\mathcal L}(\sigma)).
\end{align*}
The misclassification error cannot be smaller than the {\em Bayes error,} which is the infimum (in fact, the minimum) of the errors of all the classifiers $T$ defined on $\Omega$:
\[\ell^{\ast}=\ell^{\ast}(\tilde\mu)=\inf_{T}{\mathrm{err}}_{\tilde\mu}(T).\]

Define the measure $\mu = \tilde\mu\circ\pi^{-1}$, where $\pi$ is the first coordinate projection of $\Omega\times \{0,1\}$. This is a probability measure on $(\Omega,{\mathcal A})$. Now define a finite measure $\mu_1$ on $\Omega$ by $\mu_1(A)=\tilde\mu(A\times\{1\})$. Clearly, $\mu_1$ is absolutely continuous with regard to $\mu$. Define the {\em regression function}, $\eta\colon \Omega\to [0,1]$, as the corresponding Radon--Nikod\'ym derivative
\begin{align*}
\eta(x) & =\frac{d\mu_1}{d\mu} \\
&= P[Y=1\mid X=x],
\end{align*}
that is, the conditional probability for $x$ to be labelled $1$. (For the Radon--Nikod\'ym theorem in our abstract setting, see \cite{fremlin2}, 232E and 232B.) 

Given a classifier $T=\chi_C$, the Bayes error can be written as
\begin{equation}
\ell^{\ast} = \int_C(1-\eta)\,d\mu + \int_{\Omega\setminus C}\eta\,d\mu.
\label{eq:expressionforBayes}
\end{equation}

Now it is easy to see that the Bayes error $\ell^{\ast}$ is achieved at exactly those classifiers $T$ satisfying
\[T(x) = \begin{cases} 1,&\mbox{ for $\mu$-almost all }x\mbox{ such that }\eta(x)>\frac 12,\\
0,&\mbox{ for $\mu$-almost all }x\mbox{ such that }\eta(x)<\frac 12.
\end{cases}
\]

A rule $\mathcal L$ is {\em consistent} under $\tilde\mu$ if
\[\mbox{err}_{\tilde \mu}{\mathcal L}_n \to \ell^{\ast}(\tilde\mu),\]
and {\em universally consistent} if $\mathcal L$ is consistent under every probability measure on $\Omega\times \{0,1\}$.

Notice that since the regression function $\eta$, together with the measure $\mu$, allows to fully reconstruct the measure $\tilde\mu$, a learning problem in a measurable space $(\Omega,{\mathcal A})$ can be alternatively given either by the measure $\tilde\mu$ or by the pair $(\mu,\eta)$.

The usual setting for statistical learning is a standard Borel space as $\Omega$. In other words, the sigma-algebra $\mathcal A$ consists of all Borel sets generated by some complete separable metric on $\Omega$. However, apriori there are no restrictions for studying the learning problems in more general measurable spaces. In this note, we will concentrate on a particular situation where the sigma-algebra is the Borel sigma-algebra generated by a metric, but the metric space $\Omega$ is not necessarily separable.

\subsection{The $k$ nearest neighbour classification rule}

The $k$-NN classifier in $\Omega$ is a learning rule, defined by selecting the label ${\mathcal L}_n(\sigma)(x)\in\{0,1\}$ for a point $x$ on the basis of a labelled $n$-sample $\sigma=(x_1,x_2,\ldots,x_n,y_1,y_2,\ldots,y_n)$, $x_i\in\Omega$, $y_i\in\{0,1\}$, by the majority vote among the values of $y_i$ corresponding to the $k=k_n$ nearest neighbours of $x$ in the learning sample $\sigma$. 

There is an issue of possibly occurring ties, which come in two types. One is the voting tie, when $k$ is even and we may have a split vote. This can be broken, in fact, in any way, without affecting the consistency of the classifier. For instance, one can always choose the label $1$ (as we do below), or just assign the label in a random way.

It may also be that there are more than $k$ neighbours of $x$ within $\sigma$ that are at the same distance. This requires a tie-breaking rule. 
Given $k$ and $n\geq k$, define 
\begin{equation}
r^{\varsigma_n}_{k\mbox{\tiny -NN}}(x)=\min\{r\geq 0\colon \sharp\{i=1,2,\ldots,n\colon x_i\in \bar B_r(x)\}\geq k\}.
\label{eq:rknn}
\end{equation}
In other words, this is the smallest radius of a closed ball around $x$ containing at least $k$ nearest neighbours of $x$ in $\varsigma_n$.

A {\em $k$ nearest neighbour mapping} is a function 
\[k\mbox{-NN}^{\sigma}\colon\Omega^n\times\Omega\to\Omega^k\]
which selects a $k$-subsample $k\mbox{-NN}^{\sigma}(x) \sqsubset \sigma$ so that 
\begin{enumerate}
\item
all elements of $k\mbox{-NN}^{\sigma}(x)$ are at a distance $\leq r^{\varsigma_n}_{k\mbox{\tiny -NN}}(x)$ from $x$, and
\item
all points $x_i$ in $\sigma$ that are at a distance strictly less than $r^{\varsigma_n}_{k\mbox{\tiny -NN}}(x)$ to $x$ are in $k\mbox{-NN}^{\sigma}(x)$. 
\end{enumerate}

The $k$ nearest neighbour mapping $k\mbox{-NN}^{\sigma}$ can be deterministic or stochastic, in which case it will depend on an additional random variable, independent of the sample path. A typical example of the kind would be to give the sample $\sigma$ a random order, under a uniform distribution on the group of $n$-permutations, and break the distance ties by selecting among the tied neighbours on the sphere the smallest ones under the order selected.

Here is a formal definition of the $k$-NN learning rule:
\begin{eqnarray*}
{\mathcal L}^{k\mbox{\tiny -NN}}_n(\sigma)(x) &=&
\theta\left[\frac 1k\sum_{x_i\in k\mbox{\tiny -NN}^{\sigma}(x)}y_i-\frac 12 \right].
\end{eqnarray*}
Above, $\theta$ is the Heaviside function:
\[\theta(t) =\begin{cases} 1,&\mbox{ if }t\geq 0,\\
0,&\mbox{ if }t<0.\end{cases}\]

The $k$-NN rule was historically the first classification learning rule in a standard Borel space whose universal consistency was established, by Charles J. Stone \cite{stone}.

\begin{theorem}[C.J. Stone, 1977] The $k$-nearest neighbour classifier is universally consistent in the finite-dimensional Euclidean space whenever $k=k_n\to\infty$, $k/n\to 0$.
\end{theorem}

\section{Metric spaces and measures}

\subsection{Metric dimension}
The $k$-NN classifier is no longer universally consistent in more general separable metric spaces, in fact not even in the infinite-dimensional Hilbert space $\ell^2$. An example of this kind (constructed for the needs of real analysis) belongs to Preiss \cite{Preiss_1979}. (See this example adapted for the $k$-NN classifier in \cite{CKP}, sect. 2.) 

As far as we know, there is no known criterion characterizing those separable metric spaces in which the $k$-NN classifier is universally consistent. The most general result to date belongs to C\'erou and Guyader \cite{CG}, who have shown that the $k$-NN classifier is consistent under the assumption that the regression function $\eta(x)$ satisfies the weak Lebesgue--Besicovitch differentiation property:
\begin{equation}
\frac{1}{\mu(B_r(x))}\int_{B_r(x)} \eta(x)\,d\mu(x) \overset{r\downarrow 0}\longrightarrow \eta(x),
\label{eq:lb}
\end{equation}
where the convergence is in measure. Thus, the $k$-NN classifier is universally consistent in a metric space $X$ which has the weak Lebesgue--Besicovitch differentiation property with regard to every Borel sigma-finite (equivalentely, probability) measure. 

Characterizing such metric spaces also apparently remains an open problem, as mentioned in \cite{preiss83}. However, the {\em complete} metric spaces with the {\em strong} Lebesgue--Besicovitch differentiation property (that is, the convergence in Eq. (\ref{eq:lb}) is almost everywhere) have been completely described by Preiss \cite{preiss83}. Let us first state the relevant notions of metric dimension in the sense of Nagata and Preiss.

Recall that a family $\gamma$ of subsets of a set $\Omega$ has multiplicity $\leq\delta$ if the intersection of more than $\delta$ pairwise different elements of $\gamma$ is empty. The following definition by Preiss \cite{preiss83} is a relativization of a concept due to Nagata \cite{nagata}. 

\begin{definition} 
Let $\delta$ be a natural number. Say that a metric subspace $X$ of a metric space $\Omega$ has {\em Nagata dimension $\leq\delta$ on the scale $s>0$ inside of} $\Omega$ if every finite family of closed balls in $\Omega$ with centres in $X$ admits a subfamily of multiplicity $\leq\delta+1$ in $\Omega$ which covers all the centres of the original balls. The subspace $X$ has a finite Nagata dimension in $\Omega$ if $X$ has finite dimension in $\Omega$ on some scale $s>0$. Notation: $\dim^s_{Nag}(X,\Omega)$ or sometimes simply $\dim_{Nag}(X,\Omega)$. If $\Omega$ has Nagata dimension $\delta$ as a subset of itself, one says simply that the Nagata dimension of $\Omega$ is $\delta$.
\label{d:dimnagata}
\end{definition}

\begin{example}
It is not difficult to see that a metric space has Nagata dimension $0$ on the scale $+\infty$ if and only if it is non-archimedian, that is, the metric satisfies the strong triangle inequality:
\[d(x,z)\leq \max\{d(x,y),d(y,z)\}.\]
In particular, if $\Omega$ is a set equipped with a metric only taking two values, then $\dim_{Nag}(\Omega)=0$.
\end{example}

\begin{example}
The Nagata dimension of the real line is $1$, while for the plane it is $5$. (See \cite{nagata_open}, and \cite{CKP}, Example 4.4 and Remark 4.7 for calculations.) More generally, the Nagata dimension of the Euclidean space $\ell^2(d)$ is finite for all $d$, but it is unclear if the exact values of dimension have ever been calculated (a ``possibly open question'' in \cite{nagata_open}).
\end{example}

\begin{example}
\label{ex:cerou-guyader_dimension}
The Nagata dimension of the example of C\'erou and Guyader described in the Introduction is zero on the scale $s=1/2$. More generally, any uniformly discrete metric space (that is, a space where the distances between any two distinct point are uniformly bounded away from zero) has Nagata dimension zero.
\end{example}

\begin{example}
Any finite-dimensional normed space $E$ has finite Nagata dimension. The argument is based on a suitable extension of the geometric Stone lemma from the Euclidean case \cite{stone} to the case of an arbitrary norm (\cite{duan}, Lemma 2.2.9). Cover the unit sphere of $E$ with finitely many open balls $B_i$, $i=1,\ldots,k$ each having diameter $\leq 1$. For every ball $B_i$, form the corresponding cone $C_i=\{tx\colon t\in [0,1],~x\in B_i\}$ with apex at zero. Given finitely many closed balls containing zero, for every $i$ choose among them a ball whose centre has the greatest norm among all centres contained in $C_i$. Simple geometric considerations show that this ball will contain all the centres falling inside the cone $C_i$. Thus, the $k$ selected balls cover the centres of all the original balls. Now, given any finite family of balls, we can apply the same procedure to any point of $E$ contained in more than $k$ balls, to eliminate all but $k$ of them. After finitely many steps, we arrive at a subfamily of balls of multiplicity $\leq k$ that still contains the centres of all the original balls. 
\end{example}

\begin{example}
\label{ex:fdnsNd}
Moreover, a finite dimensional normed space $E$ has finite Nagata dimension in any bigger normed space $F$, independent of the ambient space: $\dim_{Nag}(E,F)=\dim_{Nag}(E)$. This follows from the geometric observation that two balls with centres in $E$ intersect in $F$ if and only if they intersect in $E$ (a consequence of the Hahn-Banach theorem).
\end{example}

Here is another notion due to Preiss \cite{preiss83}.

\begin{definition}
A metric space $\Omega$ is {\em sigma-finite dimensional in the sense of Nagata} if $\Omega=\cup_{i=1}^{\infty}X_n$, where every subspace $X_n$ has finite Nagata dimension in $\Omega$ on some scale $s_n>0$ (where the scales $s_n$ are possibly all different).
\end{definition}

Thus, the point here is not only that $\Omega$ is a union of countably many finite-dimensional subspaces, but that they must be relatively finite dimensional inside of $\Omega$.

\begin{example}
Every countable metric space is sigma-finite dimensional, because each finite subset $X$ of a metric space $\Omega$ has a finite relative Nagata dimension in $\Omega$, not exceeding $\abs X-1$.
\end{example}

\begin{example}
\label{ex:hedgehog}
Let $\tau$ be any cardinal.
By $J(\tau)$ we denote the metric hedgehog of spininess $\tau$, that is, a sum of $\tau$ copies of the unit interval glued together at the left endpoint and equipped with the geodesic metric (the maximal metric coinciding with the usual Euclidean distance on each copy of $[0,1]$). See \cite{engelking}, Ex. 4.5.1, page 251.

The Nagata dimension of the hedgehog is infinite on any scale $s>0$: indeed, the family of balls centred at all points on the ``spines'' at a distance say $s/2$ from zero, of radius $s/2$, admits no proper subfamily containing the centres, yet has infinite multiplicity.
At the same time, the hedgehog is sigma-finite dimensional, being the union of the singleton $\{0\}$ and the complements to the open $1/n$-balls around zero, for $n=1,2,\ldots$. Indeed it is easy to see that every such complement has Nagata dimension $1$ on the scale $s=1/n$.
\end{example}

\begin{example}
\label{ex:c00}
Let $\Gamma$ be any set. The normed space $c_{00}(\Gamma)$ consists of all maps $x\colon\Gamma\to\R$ such that the set $\mbox{supp}\,x =\{\gamma\in\Gamma\colon x_{\gamma}\neq 0\}$ is finite. The norm is the supremum norm, which in this case is the maximum $\norm {x}_{\infty}=\max_{\gamma\in\Gamma}\abs{x_{\gamma}}$.
In particular, $c_{00}(\Gamma)$ contains an isometric copy of the hedgehog $J(\tau)$, where $\tau=\abs\Gamma$, as a set of all sequences $x$ satisfying $\abs{\mbox{supp}\,x }\leq 1$ and $0\leq x_{\gamma}\leq 1$ for all $\gamma$.

For every $m,n\geq 1$, define 
\[X_{m,n}=\{x\in c_{00}(\Gamma)\colon\abs{\mbox{supp}\,x}=m \wedge \forall\gamma\in \mbox{supp}\,x,~\abs{x_{\gamma}}\geq 1/n\}.\]
If now $x\in X_{m,n}$ and $r<1/n$, then the closed $r$-ball around $x$, formed in $X_{m,n}$, is entirely contained in the vector space $c_{00}(\mbox{supp}\,x)$. Thus, $X_{m,n}$ is the union of disjoint isometric copies of a certain subset of the normed space $c_{00}(m)=\ell^{\infty}(m)$, with the distances between two different copies limited from below by $2/n$. Now it easily follows from Example \ref{ex:fdnsNd} that the Nagata dimension of $X_{m,n}$ in $c_{00}(\Gamma)$ on the scale $s=1/n$ is the same as that of the normed space $\ell^{\infty}(m)$, hence finite. Since
\[c_{00}(\Gamma) = \{0\}\cup\bigcup_{m,n=1}^{\infty}X_{m,n},\]
we conclude that the metric space $c_{00}(\Gamma)$ is sigma-finite dimensional.

At the same time, no infinite-dimensional Banach space is sigma-finite dimensional, nor even metrizable with a sigma-finite dimensional metric \cite{preiss83}.
\end{example}

\begin{theorem}[D. Preiss \cite{preiss83}]
For a complete separable metric space $\Omega$, the following conditions are equivalent.
\begin{enumerate}
\item $\Omega$ has the {\em strong} Lebesgue--Besicovitch differentiation property with regard to every sigma-finite locally finite Borel measure,
\item $\Omega$ is sigma-finite dimensional in the sense of Nagata.
\end{enumerate}
\end{theorem}

Equivalently, one can only consider in (1) the Borel probability measures. 

Thus, the following result holds.

\begin{theorem}[C\'erou--Guyader \cite{CG}]
\label{th:cerou_guyader}
The $k$-nearest neighbour classifier is universally consistent in every complete separable metric space that is sigma-finite dimensional in the sense of Nagata.
\end{theorem}

It is perhaps worth noting that the proof of Preiss was just a brief sketch, and for finite dimensional spaces in the sense of Nagata, the proof, in the sufficiency direction, was elaborated by Assouad and Quentin de Gromard in \cite{AG}. They did not assume the completeness, and established the result for a wider class of distances than just metrics.

A direct proof of the above theorem along the same lines as the original proof by Charles Stone \cite{stone}, and without the completeness assumption, can be found in \cite{CKP}.

\subsection{Borel measures on metric spaces}

Let us recall that a cardinal number $\tau$ is called {\em real-valued measurable} if a set of cardinality $\tau$ admits a sigma-additive probability measure defined on the family of all subsets and such that the measure of every singleton is zero, and the ideal of negligeable sets is closed under unions of $<\tau$ members. In particular, a cardinal $\tau$ admits a probability measure vanishing on singletons if and only if $\tau$ is greater than or equal to some real-valued measurable cardinal. It is relatively consistent with ZFC to assume that real-valued measurable cardinals do not exist. At the same time, the consistency of the existence of real-valued measurable cardinals cannot be proved in ZFC. 

The support of a probability measure $\mu$ on a metric space $\Omega$ is the closed set
\[\mbox{supp}\,\mu =\{x\in\Omega\colon \forall \ve>0,~\mu(B_{\ve}(x))>0\}.\]
The {\em density} of a metric space $\Omega$, denoted $d(\Omega)$, is the smallest cardinality of a dense subset of $\Omega$.

We will say that a Borel probability measure $\mu$ on a metric space $\Omega$ is {\em regular} if for each Borel subset $B\subseteq\Omega$ and every $\ve>0$ there is a closed precompact subset $K\subseteq B$ with $\mu(B\setminus K)<\ve$.

The following is well known, but we could not find a single reference where the equivalences were put together.

\begin{theorem}
\label{th:measures_metric_space}
Let $X$ be a metric space. The following conditions are equivalent.
\begin{enumerate} 
\item\label{rvm:1}
The support of every Borel probability measure $\mu$ on $X$ has full measure.
\item \label{rvm:2}
Every Borel probability measure on $X$ is supported on a separable subspace (that is, there is a separable subspace of measure one).
\item\label{rvm:2a}
Every Borel probability measure on $X$ is regular (as defined above).
\item\label{rvm:3} 
The density of $X$ is strictly less than any real-valued measurable cardinal.
\end{enumerate}
\end{theorem}

\begin{proof}
(\ref{rvm:1})$\Rightarrow$(\ref{rvm:2}): Fix $\ve>0$ and $\delta>0$. Under the assumption (\ref{rvm:1}), there is a finite subset $F({\ve},\delta)$ whose $\delta$-neighbourhood, $F({\ve},\delta)_{\delta}$, formed in $X$ has measure $>1-\ve$. The set $\cup_{m,n\geq 1}F({1/n},{1/m})$ is countable and its closure in $X$ has full measure:
\begin{align*}
\overline{\cup_{m,n\geq 1}F({1/n},{1/m})}&=
\bigcap_{k\geq 1}\left(\cup_{m,n\geq 1}F({1/n},{1/m})\right)_{1/k} 
\\
&\supseteq \bigcap_{m\geq 1}\left(\cup_{n\geq 1}F({1/n},{1/m}) \right)_{1/m},
\end{align*}
and the latter set has measure $1$.

(\ref{rvm:2})$\Rightarrow$(\ref{rvm:2a}): this implication reduces to a classical result in the case of complete separable metric spaces, in which case $K$ can be chosen compact. Suppose $Y=\mbox{supp}\,\mu$ is separable. View $\mu$ as a Borel probability measure on the completion $\hat Y$ via the rule $\hat\mu(A)=\mu(A\cap Y)$ for each Borel $A\subseteq Y$. There is a compact $K\subseteq \hat Y$ with $\mu(K)>1-\ve/2$. Let $C\subseteq Y$ be Borel. By another standard regularity result, there is a subset $F\subseteq C$ closed in $Y$ with $\mu C\setminus Y)<\ve/2$. The subset $F\cap K$ is precompact, closed in $Y$, and satisfies $\mu(C\setminus (F\cap K))<\ve$.

(\ref{rvm:2a})$\Rightarrow$(\ref{rvm:1}): by assumption, there is a sequence of precompact subsets whose union, $Y$ has full measure. This $Y$ is a separable metric subspace of full measure. The set $Y\setminus \mbox{supp}\,\mu$ can be covered by countably many balls of measure zero each, hence is a null set. We conclude: $\mu(\mbox{supp}\,\mu)\geq \mu(Y)=1$.

(\ref{rvm:2})$\Rightarrow$(\ref{rvm:3}): Contraposition. Select a subset $Y\subseteq X$ with the property that $\tau=\abs Y$ is real-valued measurable. Let $\mu$ be a probability measure defined on the sigma-algebra of all subsets of $Y$ and vanishing on singletons. Extend $\mu$ to a probability measure on the Borel sigma-algebra of $X$ (indeed, on the sigma-algebra of all subsets of $X$) by letting $\mu(A)=\mu(A\cap Y)$ for every $A\subseteq X$. For every separable subset $Z\subseteq X$ one has $\mu(Z) = \mu(Z\cap Y) =\sum_{z\in Z\cap Y}\mu(z)=0$.

(\ref{rvm:3})$\Rightarrow$(\ref{rvm:1}): by contraposition. We follow the argument in \cite{fremlin}, page 59, proof of Th. 6M. Suppose $\mu$ is a Borel probability measure on $X$ having the property $\mu(\mbox{supp}\,(\mu))<1$. Set $U=X\setminus \mbox{supp}\,(\mu)$. As each point of $U$ has a neighbourhood of zero measure, we can assume, by passing to a smaller open subspace if necessary, that there is $\ve>0$ with the property that for each $x\in U$, $\mu(B_{\ve}(x))=0$.
According to the Bing metrization theorem (see e.g. \cite{engelking}, Sect. 5.4), there is a sequence ${\mathcal O}_n$ of discrete families of open sets (each point $x\in\Omega$ has a neighbourhood only meeting at most one element of ${\mathcal O}_n$) whose union is a topological base for $U$. Clearly, we can assume that every element of this base is contained in one of the open balls of radius $\ve$.
For each $n$, define $U_n = \cup\{V\in {\mathcal O}_n\colon V\subseteq U\}$. Since $\cup U_n=U$, for some $n$ we have $\mu(U_n)>0$. Now define a measure on the sigma-algebra of all subsets of ${\mathcal O}_n$ by
\[\nu(\mathcal V) = \mu(\cup {\mathcal V})\mu(U_n)^{-1}.\]
This is a probability measure defined on all subsets of ${\mathcal O}_n$ and vanishing on singletons. Therefore, the cardinality of ${\mathcal O}_n$ is greater than or equal to a real-valued measurable cardinal. At the same time, this cardinality cannot exceed the density of $\Omega$.
\end{proof}

\section{Learning and non-separability}

If $(\Omega,{\mathcal A})$ is a measurable space, then a {\em measurable subspace} of $\Omega$ will denote any subset $Y\in {\mathcal A}$ (not necessarily measurable), equipped with the sigma-algebra 
\[{\mathcal A}\vert_Y=\{A\cap Y\colon A\in {\mathcal A}\}.\]
If now $\mathcal L$ is a learning rule for $\Omega$, we define its restriction ${\mathcal L}\vert_Y$ to $Y$ in an obvious way: for every $n\geq 1$,
\[({\mathcal L}\vert_Y)_n = {\mathcal L}_n\vert_{Y^n\times \{0,1\}^n\times Y}.\]

In the context of learning in the classical standard Borel space setting, the following result is something so obvious that it will probably be never stated explicitely and proved. Since we are working in the context of general measure spaces, it is better to verify all the details.

\begin{lemma}
\label{l:subspace}
Let $(\Omega,{\mathcal A})$ be a measurable space admitting a universally consistent learning rule $\mathcal L$. Then the restriction of $\mathcal L$ to any subset $Y\subseteq\Omega$ is a universally consistent learning rule for the measurable subspace $(Y,{\mathcal A}\vert_Y)$.
\end{lemma}

\begin{proof}
Let $Y=(Y,{\mathcal A}\vert_Y)$ be a measurable subspace of $\Omega$, and $\mu_Y$ a probability measure on $Y$. Define a probability measure $\mu_{\Omega}$ on $\Omega$: if $A\in {\mathcal A}$, then 
\[\mu_{\Omega}(A) = \mu_Y(A\cap Y).\]
Notice that the outer measure $\mu_{\Omega}^{\ast}(B)$ of every element $B$ of ${\mathcal A}_Y$ equals $\mu_Y(B)$. Consequently, the set $\Omega$ forms a {\em measurable envelope} of $Y$ in the measure space $(\Omega,{\mathcal A},\mu_{\Omega})$ in the sense of \cite{fremlin1}, 132D. Namely: $Y\subseteq\Omega$ and for every $A\in {\mathcal A}$, one has $\mu_{\Omega}(A\cap \Omega)=\mu_{\Omega}^{\ast}(A\cap Y)$.

Let $\eta_{Y}\colon Y\to [0,1]$ be a ${\mathcal A}_Y$-measurable function (a regression function for a learning problem in $Y$). There exists a measurable extension, $\eta_{\Omega}$, of $\eta_Y$ over $(\Omega,{\mathcal A})$, such that 
$\int_A\eta_{\Omega}\,d\mu_{\Omega}=\int_{A\cap Y}\eta_Yd\mu_Y$ for every $A\in {\mathcal A}$ (see \cite{fremlin2}, Proposition 214E(b)). Now it follows from Eq. (\ref{eq:expressionforBayes}) that the learning problems $(\mu_Y,\eta_Y)$ in $(Y,{\mathcal A}_Y)$ and the learning problem $(\mu_\Omega,\eta_\Omega)$ in $(\Omega,{\mathcal A})$ have the same Bayes error.

Now define a probability measure $\tilde\mu_{\Omega}$ on the measure space $\Omega\times\{0,1\}$ (with a canonical product sigma-algebra) from the pair $(\mu_{\Omega},\eta_{\Omega})$ by
\[\tilde\mu_{\Omega}(A) = \int_{A\cap \Omega\times\{1\}}\eta_{\Omega}\,d\mu_{\Omega} +
\int_{A\cap \Omega\times\{0\}}(1-\eta_{\Omega})\,d\mu_{\Omega}.\]
In a similar way, define a probability measure $\tilde\mu_Y$ on $Y\times \{0,1\}$:
\[\tilde\mu_{Y}(A) = \int_{A\cap Y\times\{1\}}\eta_{Y}\,d\mu_{Y} +
\int_{A\cap Y\times\{0\}}(1-\eta_{Y})\,d\mu_{Y}.\]
It follows from the choice of $\eta_{\Omega}$ that for every measurable $A\subseteq\Omega\times\{0,1\}$, one has $\tilde\mu_{\Omega}(A) = \tilde\mu_Y(A\cap (Y\times\{0,1\}))$, and in particular $\Omega\times\{0,1\}$ is a measurable envelope of $Y\times\{0,1\}$ with regard to the measure $\tilde\mu_{\Omega}$. 

Let $n\in\N_+$. The product space $\Omega^n\times\{0,1\}^n\times\Omega\times\{0,1\}$ with the product measure $\tilde\mu_{\Omega}^{n+1}$ is, by the same argument, a measurable envelope of the subset $Y^n\times\{0,1\}^n\times Y\times\{0,1\}$ with the product measure $\tilde\mu_{Y}^{n+1}$.

Now let $\mathcal L$ be a learning rule for $(\Omega,{\mathcal A})$ consistent under the problem $(\mu_{\Omega},\eta_{\Omega})$. For every $n$, the function
\[\Omega^n\times\{0,1\}^n\times\Omega\times \{0,1\} \ni (\sigma,x,y)\mapsto \abs{{\mathcal L}_n(\sigma)(x)-y}\in\{0,1\}\]
is measurable, so
we have
\begin{align*}
\mbox{err}_{\tilde\mu_Y}{\mathcal L}\vert_Y & = (\tilde\mu_Y^n\otimes \tilde\mu_Y)\{(\sigma,x,y)\in Y^n\times \{0,1\}^n\times Y\colon {\mathcal L}_n\vert_Y(\sigma)(x)\neq y\}
 \\
&= (\tilde\mu_{\Omega}^n\otimes \tilde\mu_{\Omega})\{(\sigma,x,y)\in \Omega^n\times\{0,1\}^n\times\Omega\colon {\mathcal L}_n(\sigma)(x)\neq y\}\\
&=  
\mbox{err}_{\tilde\mu_{\Omega}}{\mathcal L}_n\\
&\overset{n\to\infty}\longrightarrow \ell^{\ast}(\tilde\mu_{\Omega},\Omega) \\
&= \ell^{\ast}(\tilde\mu_Y,Y). \\
\end{align*}
\end{proof}

Let $(X,{\mathcal A},\mu)$ be a probability space. The distance between two elements $A,B\in {\mathcal A}$, given by the expression $\mu(A\Delta B)$, is a pseudometric, and the associated metric space is ${\mathcal A}/{\mathcal N}_{\mu}$, the quotient of the Boolean algebra $\mathcal A$ (with the natural operations) by the ideal of null sets, ${\mathcal N}_{\mu}$. The algebra ${\mathcal A}/{\mathcal N}_{\mu}$ equipped with the metric corresponding to the pseudometric $\mu\circ\Delta$ is called the {\em measure algebra} of the measure space $X$. We will still denote it $({\mathcal A},\mu\circ\Delta)$.
(In fact, we will not be interested in the Boolean algebra structure.) The {\em Maharam type} of the measure algebra can be defined as the smallest cardinality of a subset generating a dense subalgebra. In particular, the Maharam type does not exceed the density of the underlying metric space. For an in-depth treatment of the Maharam type, see \cite{fremlin3}, Ch. 33.

\begin{theorem}
\label{th:maharam_type}
Let $(\Omega,{\mathcal A},\mu)$ be a probability space admitting a learning rule consistent for every measurable regression function. Then the Maharam type of the measure algebra $({\mathcal A},\mu\circ\Delta)$ does not exceed the cardinality of $\Omega$.
\end{theorem}

\begin{proof}
Let $A\in {\mathcal A}$ be any concept. Applying our assumption to the regression function $\eta=\chi_A$, we conclude that for every $\ve>0$ there should exist at least one labelled sample $\sigma$ generating a hypothesis learning $A$ to a precision $<\ve$:
\[\mbox{err}_{\mu,\eta}({\mathcal L}(\sigma)) =\mu({\mathcal L}(\sigma)\Delta A)<\ve.\]
Consequently, the rule $\mathcal L$, viewed as a mapping
\[{\mathcal L}\colon \bigcup_{n=1}^{\infty} \Omega^n\times \{0,1\}^n\to {\mathcal A},\]
must have a dense image in the measure algebra with regard to the distance $\mu\circ\Delta$. At the same time, the cardinality of the set of all labelled samples equals the cardinality of $\Omega$, implying the result.
\end{proof}

\begin{corollary}
\label{c:atomlessly_no_rule}
Let $\tau$ be an atomlessly measurable cardinal, with a witnessing probability $\mu$. The measure space $(\tau,2^{\tau},\mu)$ admits no learning rule that is consistent for every measurable regression function.
\end{corollary}

\begin{proof} 
According to the Gitik--Shelah theorem \cite{GH}, for an atomlessly measurable cardinal $\tau$, the Maharam type of the measure algebra $(2^{\tau},\mu\circ\Delta)$ is strictly greater than $\tau$ (in fact, at least $\min\{\tau^{(+\omega)},2^{\tau}\}$, see \cite{fremlin3}, page 20, Th. 3F). 
\end{proof}

\begin{corollary}
Let $\Omega$ be a metric space whose density is greater than or equal to an atomlessly measurable cardinal. Then the $k$ nearest neighbour learning rule in $\Omega$ is not universally consistent.
\end{corollary}

\begin{proof}
Suppose $d(\Omega)\geq \tau$, where $\tau$ is atomlessly measurable. Since $\tau$ is regular and uncountable, its cofinality type is uncountable too, and in particular there exist $\ve>0$ and a $\ve$-discrete subset $Y$ of $\Omega$ with $\abs Y\ = d(\Omega)\geq \tau$. The Borel structure of $Y$ is the sigma-algebra $2^Y$ of all subsets of $Y$.
Let $\mu$ be an atomless probability measure defined on $2^Y$.
By Corollary \ref{c:atomlessly_no_rule}, the $k$-NN classifier is not consistent in the space $(Y,\mu)$. We conclude by Lemma \ref{l:subspace}, because the restriction of the $k$-NN learning rule from $\Omega$ to $Y$ is the $k$-NN learning rule in $Y$.
\end{proof}

The case of two-valued measurable cardinals will be based on the following slight modification of the example by C\'erou and Guyader.

\begin{proposition}
\label{p:two-valued_non-consistent}
Let $\Omega$ be a metric space whose metric only takes two values, $\{0,r\}$, $r>0$.
The $k$-NN classifier (under the uniform distance tie-breaking) is universally consistent in $\Omega$ if and only if the cardinality of $\Omega$ is strictly less than any real-valued measurable cardinal. 
\end{proposition}

\begin{proof}
The necessity is proved by contraposition. Suppose $\abs\Omega\geq\tau$, where $\tau$ is real-valued measurable. Equip $\Omega$ with a probability Borel measure, $\mu$ (that is, a probability measure defined on all subsets) such that the singletons are null sets. Fix any point $y\in \Omega$. Now define a measure
\[\mu = \frac 13 \nu + \frac 23 \delta_{y},\]
and a regression function,
\[\eta = \chi_{\{y\}}.\]
In other words, with probability $1/3$ a random point $X$ will be equal to $y$ and labelled deterministically $1$, and with probability $2/3$ it will follow the law $\nu$, to be labelled $0$. 

For an i.i.d. sample $(X_1,\ldots,X_n)$ distributed according to $\mu$, the random variables $\eta(X_n)$ are i.i.d. Bernoulli random variables with the probability of success $2/3$. Given a point $x\in\Omega$, all elements of the sample are at the same distance from $x$, so the choice of $k$ nearest neighbours of $x$ is itself a random variable with values in the set $[\sigma]^k$ of all $k$-subsets of $\sigma$, following a uniform distribution. Equivalently, if $\sigma=(X_1,\ldots,X_n)$, then the $k$ nearest neighbours are the random variables $X_{\tau(1)},X_{\tau(2)},\ldots,X_{\tau(n)}$, where $\tau\in S_n$ is a random permutation. Since the i.i.d. random variables are exchangeable, that is, the joint law does not change under any permutation, it follows that the $k$ nearest neighbours are i.i.d., following the law $\mu$. 

According to the Law of Large Numbers with Chernoff bounds, 
\begin{align*}
P[\E\{\sharp i=1,2,\ldots,k\colon X_{\tau(i)}= y\}>1/2 ] &=
P\left[\frac 1n\sum_{i=1}^k Y_{\tau(i)} > \frac 12\right] \\
&\geq 1 - \exp(-k/16).
\end{align*}
Thus, with exponentially high confidence, at least a half of the $k$ nearest neighbours of $x$ will be equal to $y$, and consequently $x$ will be labelled $1$. Now we have, using the Fubini theorem ((\cite{fremlin2}, Th. 252B),
\begin{align*}
\mbox{err}_{\mu,\eta}L_n^{k\mbox{\tiny -NN}} &= \mu(\{y\}\Delta \{x\colon L_n^{k\mbox{\tiny -NN}}(X) =1\})\\
&\geq \mu(\{x\colon L_n^{k\mbox{\tiny -NN}}(X) =1\}\setminus \{y\})\\
&=
P[L_n^{k\mbox{\tiny -NN}}(X) =1]-\frac 23\\
 &= \int_{\Omega}P[L_n^{k\mbox{\tiny -NN}}(X) =1\mid X=x]\,d\mu(x)-\frac 23 \\
&\geq \frac 13 - \exp(-k/18) \\
&\to \frac 13\mbox{ when }n\to\infty.
\end{align*}
At the same time, the Bayes classifier is equal to the deterministic regression function $\eta=\chi_{\{y\}}$, so the Bayes error is zero.

The sufficiency ($\Leftarrow$) follows from Theorem \ref{th:measures_metric_space}. Condition (\ref{rvm:2}) says that under our assumption, for any probability measure $\tilde\mu$ on $\Omega\times \{0,1\}$ there is a countable subspace $\Omega^\prime\subseteq\Omega$ such that $\Omega^\prime\times \{0,1\}$ has full measure, and so the $k$-NN classifier is consistent. This follows of course from Theorem \ref{th:cerou_guyader} of C\'erou and Guyader (and Preiss), since a countable metric space is sigma-finite dimensional. But there is no need to use such a strong result, because the measure $\tilde\mu$ is purely atomic, and so asymptotically the $k$ nearest neighbours of each point $x\in\mbox{supp}\,\mu$ will all be equal to $x$ with exponentially high confidence. Now an application of the Law of Large Numbers assures that, whenever $\eta(x)\neq 1/2$, the label of $x$ will converge to the label predicted by the Bayes classifier as $n\to\infty$.
\end{proof}

To handle the case of two-valued measurable cardinals, we need some combinatorial properties of such cardinals. A two-valued probability measure $\mu$ defined on a cardinal $\tau$ is {\em normal} if, whenever the subsets $X_\alpha\subseteq\tau$, $\alpha<\tau$ have measure one, the set
\[\{\alpha<\tau\colon \alpha\in\cap_{\beta<\alpha}X_{\beta}\} \]
(the {\em diagonal intersection} of $X_{\alpha}$)
has measure one. Equivalently (\cite{jech}, Exercise 8.8), $\mu$ is normal if and only if every regressive function $f\colon\tau\to\tau$ (that is, $f(\xi)<\xi$ for all $\xi<\tau$) is constant on a set of full measure.
Every two-valued measurable cardinal admits a normal witnessing measure, that is, a two-valued probability measure $\mu$ defined on all subsets, vanishing on singletons, whose ideal of null sets is closed under unions of $<\tau$ members, and which is in addition normal in the sense of the above definition. (See \cite{jech}, Th. 10.20, or \cite{fremlin}, Th. 1G.)

For a set $X$ we denote, in the usual combinatorial notation, by $[X]^{<\omega}$ the family of all finite subsets of $X$.
Suppose $\tau$ is a two-valued measurable cardinal, and $\mu$ a normal witnessing measure on $\tau$. Then we have the following Ramsey-type property: for every colouring of the set $[\tau]^{<\omega}$ with fewer than $\tau$ colours, there is a subset $X\subseteq\tau$ of $\mu$-measure one such that for every $n$, $[X]^{n}$ is monochromatic. (Theorem by Rowbottom \cite{rowbottom}; cit. by \cite{jech}, Th. 10.22.) In particular, the cardinality of $X$ is $\tau$.

\begin{lemma}
\label{l:exists_subspace}
Let $\Omega$ be a metric space, $\tau$ a two-valued measurable cardinal, and $\abs{\Omega}\geq\tau$. Then there exist $r>0$ and a metric subspace $Y\subseteq\Omega$ of cardinality $\tau$, on which the metric only takes the values $0$ and $r$.
\end{lemma}

\begin{proof}
Select a subset $Z\subseteq\Omega$ of cardinality $\tau$ and colour $[Z]^{2}$ with non-negative reals, by setting the colour of a couple $\{x,y\}\in [Z]^2$ to be $d(x,y)$. Since $\mathfrak c<\tau$, Rowbottom's theorem implies the existence of a subset $Y\subseteq Z$ of cardinality $\tau$ with the desired properties. 
\end{proof}

\begin{theorem}
\label{th:two-valued_non}
Let $\tau$ be a two-valued measurable cardinal, and let $\Omega$ be a metric space with $\abs{\Omega}\geq\tau$. Then the $k$-NN classifier is not universally consistent in $\Omega$.
\end{theorem}

\begin{proof}
By Lemma \ref{l:exists_subspace} there is a metric subspace $Y$ of cardinality $\tau$ with a two-valued metric. According to proposition \ref{p:two-valued_non-consistent}, the $k$-NN classification rule is not consistent in $Y$, and since the restriction of the $k$-NN classifier from $\Omega$ to $Y$ is the $k$-NN classifier in $Y$, we conclude by Lemma \ref{l:subspace}.
\end{proof}

\begin{theorem}
\label{th:central}
Let $\Omega$ be a metric space. The $k$-nearest neighbour classifier is universally consistent in $\Omega$ if and only if it is universally consistent in every separable subspace of $\Omega$, and the density of $\Omega$ is less than any real-valued measurable cardinal.
\end{theorem}

\begin{proof}
Necessity follows from Theorem \ref{th:two-valued_non} and Lemma \ref{l:subspace}, and sufficiency follows from Theorem \ref{th:measures_metric_space}.
\end{proof}

\begin{corollary}
\label{c:sigma-finite}
Let $\Omega$ be a metric space each of whose separable subspaces is sigma-finite dimensional in the sense of Nagata and Preiss. Then the $k$-NN classifier is universally consistent in $\Omega$ if and only if the density of $\Omega$ is less than the smallest real-valued measurable cardinal.
\end{corollary}

\begin{example}
\label{ex:hedgehog2}
The $k$-NN classifier is universally consistent in a metric hedgehog $J(\tau)$ of spininess $\tau$ (Example \ref{ex:hedgehog}) if and only if $\tau$ is less than any real-valued measurable cardinal.
\end{example}

\begin{example}
Let $\Gamma$ be a set. The $k$-NN classifier is universally consistent in the  normed space $c_{00}(\Gamma)$ (Example \ref{ex:c00}) if and only if $\abs\Gamma$ is strictly less than any real-valued measurable cardinal.
\end{example}

\begin{remark}
We do not know whether every metric space all of whose separable subspaces are sigma-finite dimensional is itself sigma finite-dimensional. Probably it is not the case, but we do not have any example.
\end{remark}

\section*{Addendum: Strengthening a result of Hart \cite{hart}}

This is a comment on the recent work by Ben-David, Hrube{\v s}, Moran, Shpilka, and Yehudayoff (\cite{BHMSY1}, \cite{BHMSY2}), only indirectly related to the main body of the paper (see the Introduction).
The Expectation Maximization (EMX) problem calls to guess, probably approximately correctly, a set $S(\sigma)$ having a nearly full measure on the basis of a random finite unlabelled sample, $\sigma$. The error and confidence of the guess are supposed to be uniformly bounded over a given family of probability measures on the domain (a measurable space). 

In the specific version of the problem considered by the authors, the domain is just any set, $X$, equipped with its full sigma-algebra of subsets, and the family of probability measures in question, $P_a(X)$, consists of all purely atomic measures on $X$. Denote $[X]^{<\infty}$ the family of all finite subsets of $X$. The question is: does there exist a map,
\[S\colon \bigcup_{n=1}^{\infty} X^n\to [X]^{<\infty},\]
with the property that for every $\ve,\delta>0$ there is $n=n(\ve,\delta)$ so that
\[\forall m\geq n,~P[\mu(S(\sigma))>1-\ve]>1-\delta?\]
The central theorem of \cite{BHMSY1,BHMSY2} states that such an $S$ exists if and only if the cardinality of $X$ is less than $\aleph_{\omega}$. In particular, the domain of real numbers, $\R$, admits a solution to the EMX problem over purely non-atomic measures if and only if the continuum equals $\aleph_n$ for some natural $n$, and thus the assertion is independent of ZFC.

The main criticism of the result belongs to Hart \cite{hart}, who has in particular shown that, in the most interesting case $X=\R$, no Borel measurable map $S$ with the above properties can exist. Here we will notice that a map $S$ having those properties cannot even be Lebesgue measurable.

Let $[\I]^m$ denote, for $m\in\N$, the family of all $m$-subsets of the interval given the Vietoris topology. Thus, two finite sets $A$ e $B$ with $m$ elements are $\ve$-close if $A$ is included in the $\ve$-neighbourhood of $B$ e vice versa. Let $\kappa\colon [\I]^{m+1}\to [\I]^m$ be a Lebesgue measurable map having the property $\kappa(\sigma)\subseteq\sigma$. We will identify $[\I]^{m+1}$ with a subset of all elements of $\I^{m+1}$ of the form
\[x=(x_1,x_2,\ldots,x_{m+1}),~x_1< x_2< \ldots < x_{m+1}.\]
The image of $[\I]^{m+1}$ in $\I^{m+1}$ is an open $(m+1)$-simplex, having the Lebesgue measure $1/(m+1)!$. We will denote it by the same symbol, and equip with the $\ell^{\infty}$ distance (which corresponds to the Vietoris distance) until the end of the argument.

Fix any point $x=(x_1,\ldots,x_{m+1})\in [\I]^{m+1}$, and define
\[\ve = \frac 13 \min_{1\leq i<j\leq m+1} d(x_i,x_j)>0.\]

Let $\gamma>0$ be the Lebesgue measure of the open ball $B_{\ve}(x)$ taken in $[\I]^{m+1}$ (seen as a simplex with $\ell^{\infty}$ metric).
According to Luzin's theorem, there is a compact set $K\subseteq [\I]^{m+1}$ having measure $>1-\gamma$ and such that $\kappa\vert_K$ is continuous, thus uniformly continuous. Choose $\delta\leq \ve$ so small that if $\sigma,\tau\in K$ and $d(\sigma,\tau)<\delta$, then $d(\kappa(\sigma),\kappa(\tau))<\ve$.

 Denote
\[K^\prime = K\cap B_{\ve}(x).\]
The set $K^\prime$ has a strictly positive Lebesgue measure. Therefore, there exists a point $y\in K^\prime$ whose $\delta$-neighbourhood has a strictly positive Lebesgue measure (because $K^\prime$ is precompact, so can be covered with finitely many balls of this radius). Denote $K^{\prime\prime}=K^\prime\cap B_{\delta}(y)$. 

Assume without loss of generality that 
\[\kappa(y) = (y_1,y_2,\ldots,y_m),\]
that is, the coordinate $(m+1)$ is removed. (If it is another coordinate, we will just apply a permutation to the simplex and to $K^\prime$. This mapping will of course send the simplex image of $[\I]^{m+1}$ to another subsimplex of $\I^{m+1}$, but it preserves both the Lebesgue measure and the $\ell^{\infty}$-metric.)

For any $z\in B_{\delta}(y)$, we have $d(y,z)<\delta$, and so, if $z\in K^{\prime\prime}$, then $d(\kappa(y),\kappa(z))<\ve$. Consequently, for all $i$, $\kappa(z)_i\in B_{\ve}(\kappa(y)_i)$, and in particular, $\kappa(z)$ is also obtained by removing the last coordinate of $z$. We conclude: for all $z\in K^{\prime\prime}$,
\[\kappa(z) = \pi_{[1,m]}(z),\]
the coordinate projection on the first $m$ coordinates. 

By the Fubini theorem, 
\[\mu(K^{\prime\prime}) = \int_{0}^{1} \mu^{(m)}(\pi_{[1,m]}^{-1}(z)\cap K^{\prime\prime})\, d\lambda(z),\]
and since $\mu(K^{\prime\prime})>0$, for a set of points $z$ of positive measure the set $\kappa^{-1}(z)$ is infinite.

Now the argument is concluded as in \cite{hart}: if there existed a Lebesgue measurable $S\colon \I^{<\omega}\to [\I]^{<\omega}$ giving a solution to the EMX problem for the class of finite sets under all purely atomic measures, then there would exist a Lebesgue measurable finite-to-one compression function $\kappa\colon [\I]^{m+1}\to [\I]^m$, because the choice of a point to remove can be done in a Borel measurable fashion (e.g. by always removing the smallest possible point).

\section*{Concluding questions and remarks}

1. Suppose there exists a consistent learning algorithm for a probability measure space $(X,{\mathcal A},\mu)$. Does it imply that the measure algebra $({\mathcal A},\mu\circ\Delta)$ (equivalently, the Banach space $L^1(\mu)$) is separable?

2. In particular, can one consistently learn in the Loeb measure spaces \cite{loeb}?

3. Does there exist a measurable space $(X,{\mathcal A})$ whose sigma-algebra is not countably generated and which still admits a universally consistent learning rule without any additional set-theoretic assumptions? 
Any Borel space generated by a non-separable metric space admits such a rule (the $k$-NN learning rule with regard to the original metric), but only assuming that the cardinality of the space is less than the smallest real-valued measurable cardinal.

4. Does there exist a learning problem on a standard Borel space that only admits a {\em measurable} solution under additional set-theoretic assumptions on top of ZFC?
The rule for learning inside a class studied in \cite{pestov2010,pestov2013} under the assumption of the Martin Axiom (MA) can be converted into a universally consistent rule using Vapnik's Structural Risk Minimization, but it is unclear if it can be done so that the Martin Axiom becomes a necessary condition for measurability of such a rule.

5. Suppose every separable subspace of a metric space $\Omega$ is sigma-finite dimensional in the sense of Nagata and Preiss. Will $\Omega$ be itself sigma-finite dimensional? Same question, when $\Omega$ is complete.

6. The question of combinatorial characterization of separable metric spaces in which the $k$-NN classification rule is universally consistent remains open.

\section*{References}

\bibliography{learning_hedgehog}

\end{document}